\documentclass[a4paper,preprintnumbers,floatfix,superscriptaddress,pra,showpacs,notitlepage,longbibliography]{revtex4-2}
\usepackage[utf8]{inputenc}
\usepackage[T1]{fontenc}
\usepackage[sc,osf]{mathpazo}
\usepackage{amsmath, amsthm, amssymb,amsfonts,mathbbol,amstext}
\usepackage{graphicx}
\usepackage{dcolumn}
\usepackage{bm}
\usepackage{bbm}
\usepackage{hyperref}
\usepackage{mathtools}
\usepackage{comment}
\usepackage{color}
\usepackage{dsfont}
\usepackage[normalem]{ulem}
\usepackage{multirow}
\usepackage{diagbox}
\usepackage{amsmath}
\usepackage{amsthm}
\usepackage{bbm}
\usepackage{comment} 
\usepackage{marvosym}
\usepackage{xcolor}
\usepackage{appendix}
\usepackage{hyperref}

\newtheorem{theorem}{Theorem}
\newtheorem{corollary}{Corollary}[theorem]
\newtheorem{lemma}[theorem]{Lemma}
\newtheorem{observation}[theorem]{Observation}

\definecolor{applegreen}{rgb}{0.55, 0.71, 0.0}
\definecolor{darkelectricblue}{rgb}{0.03, 0.51, 0.57}
\definecolor{atomictangerine}{rgb}{1.0, 0.6, 0.4}

\newcommand{\ket}[1]{| #1 \rangle}
\newcommand{\bra}[1]{\langle #1 |}

\newcommand{\stkout}[1]{\ifmmode\text{\sout{\ensuremath{#1}}}\else\sout{#1}\fi}
\newcommand\norm[1]{\left\lVert#1\right\rVert}

\definecolor{forestgreen}{rgb}{0.13, 0.55, 0.13}

 \def\R{\mathbbm{R}}

 \DeclareMathOperator{\tr}{tr}

\begin{document}
\title{Causal inference with imperfect instrumental variables}
\author{Nikolai Miklin}
\thanks{nikolai.miklin@hhu.de}
\affiliation{International Centre for Theory of Quantum Technologies (ICTQT), University of Gdansk, 80-308 Gdansk, Poland}
\affiliation{Heinrich Heine University D{\"u}sseldorf, Universit{\"a}tsstra{\ss}e 1, 40225 D{\"u}sseldorf, Germany}

\author{Mariami Gachechiladze}
\thanks{mgachech@uni-koeln.de}
\affiliation{Institute for Theoretical Physics, University of Cologne, 50937 Cologne, Germany}

\author{George Moreno}
\affiliation{International Institute of Physics, Federal University of Rio Grande do Norte, 59078-970, P. O. Box 1613, Natal, Brazil}
\author{Rafael Chaves}
\affiliation{International Institute of Physics, Federal University of Rio Grande do Norte, 59078-970, P. O. Box 1613, Natal, Brazil}
\affiliation{School of Science and Technology, Federal University of Rio Grande do Norte, 59078-970 Natal, Brazil}

\date{\today}
\begin{abstract}
Instrumental variables allow for quantification of cause and effect relationships even in the absence of interventions. To achieve this, a number of causal assumptions must be met, the most important of which is the independence assumption, which states that the instrument and any confounding factor must be independent. However, if this independence condition is not met, can we still work with imperfect instrumental variables? Imperfect instruments can manifest themselves by violations of the instrumental inequalities that constrain the set of correlations in the scenario. In this paper, we establish a quantitative relationship between such violations of instrumental inequalities and the minimal amount of measurement dependence required to explain them. As a result, we provide adapted inequalities that are valid in the presence of a relaxed measurement dependence assumption in the instrumental scenario. This allows for the adaptation of existing and new lower bounds on the average causal effect for instrumental scenarios with binary outcomes. Finally, we discuss our findings in the context of quantum mechanics.
\end{abstract}

\maketitle 

\section{Introduction}
Inferring causal relations from data is a central goal in any empirical science. Yet, in spite of its importance, causality has remained a thorny issue. Misled by the commonplace sentence stating that ``\emph{correlation does not imply causation}'', causal inference persists in the view of many as a noble but practically impossible task. Contrary to that, however, the surge and development of the causality theory \cite{pearl2009causality,spirtes2000causation} has proven formal conditions under which cause and effect relations can be extracted.

Consider the simplest and fundamental question of deciding whether observed correlations between two variables, $A$ (also known as treatment) and $B$ (also known as effect), are due to some direct causal influence of the first over the second, or due to a common cause, a third, potentially latent (non-observable) variable $\Lambda$. Both causal models are observationally equivalent, meaning that both models can generate the same set of possible correlations observed between the target variables. Notwithstanding, causal conclusions can be reached if, instead of passively observing the events, we perform interventions \cite{pearl2009causality,Balke1997,janzing2013quantifying}. In particular, interventions on $A$ put this variable under the experimenter's control, turning it independent of any latent common cause. If after the intervention, one still observers correlations between $A$ and $B$, then it is possible to unambiguously conclude that $A$ is a cause of $B$.  Interventions, however, are often unavailable for a variety of practical, fundamental, or ethical issues.

An elegant way to circumvent such issues are the instrumental variables \cite{pearl2009causality,wright1928tariff,angrist1996identification,greenland2000introduction,rassen2009instrumental,hernan2006instruments,lousdal2018introduction,Kedagni2020}. If a proper instrument $X$, correlated with $A$ but statistically independent of $\Lambda$, can be found, then the causal effect of $A$ over $B$ can be estimated even in the absence of interventions or structural equations. Nevertheless, since the instrumental conditions depend on an unobservable variable, identifying an instrument seems to be a matter of judgment that cannot be supported solely by the data. To cope with that, instrumental inequalities have been introduced \cite{pearl1995testability,bonet2013instrumentality,poderini2020exclusivity}, constraints that should be respected by any experiment in compliance with the instrumental assumptions. Thus, the violation of instrumental inequality is an explicit proof that one does not have a proper instrument. Does that mean, however, that no causal inference at all can be made if an instrumental inequality is violated? Or can we still rely on that instrument, even though imperfect, to infer causal relations?

Motivated by these questions, we analyze in detail a generalization of the instrumental causal structure, where we drop the assumption that one has a perfect instrument. More specifically, we relax the assumption that the instrumental variable $X$ should be independent of the latent factor $\Lambda$. Considering the case where $A$ and $B$ are dichotomic and $X$ is also discrete, we derive new instrumental inequalities that take explicitly into account the correlation between $X$ and $\Lambda$. We also generalize the bounds on the average causal effect (ACE) \cite{pearl2009causality,Balke1997,janzing2013quantifying} for more general instruments. 

Finally, we make a connection with the field of quantum foundations, where violation of instrumental inequalities can appear without relaxing the measurement independence assumption~\cite{hall2016significance,Hall2020,chaves2015unifying,chaves2021causal}. Using our results, we establish the minimal measurement dependence needed in the classical instrumental scenario to explain such violations and, as a result, we analyze the robustness of instrumental tests as witnesses of non-classical behavior.

The paper is organized as follows. In Sec.~\ref{sec2} we discuss how instrumental variables can be employed to put lower bounds on the cause and effect relations between two variables. In Sec.~\ref{sec3} we discuss the violations of independence assumption and how modified instrumental inequalities and causal bounds on ACE can be derived to take that into account. In Sec.~\ref{sec4} we discuss quantum violations of the modified inequalities. In Sec.~\ref{sec5} we discuss our findings and point out interesting questions for future research. 

\textbf{Notations:} Throughout the paper, we denote random variables by capital letters $A,B,X$ and $\Lambda$, as well as $\Lambda_X$,$\Lambda_A$,$\Lambda_B$. Without loss of generally, we consider these random variables taking values in the set of non-negative integers $\mathds{Z}_{0+}$. Probability of an event $E$ is denoted as $p(E)$. We use a common shorthand notation $p(a)=p(A=a)$ to denote the probability of $A$ taking value $a$. Similar shorthand notation is used for conditional probabilities, e.g., $p(a|x) = p(A=a|X=x)$, and interventions, e.g., $p(b|\mathrm{do}(a)) = p(B=b|\mathrm{do}(A=a))$. We use one exception to this rule for probabilities of the form $p(i,j|k)$ and $p(l|\mathrm{do}(m))$, which should be read as $p(A=i,B=j|X=k)$ and $p(B=l|\mathrm{do}(A=m))$, respectively, for any $i,j,k,l,m\in \mathds{Z}_{0+}$. We also use the common notation $[n] = \{0,1,\dots,n-1\}$.

\section{Instrumental variables, instrumental inequalities and causal bounds}
\label{sec2}

Before getting into details and illustrating the power of an instrumental variable as a causal inference tool, we discuss a simple linear structural model,  $b=\beta a+ \lambda$, where $\beta$ can be understood as the strength of the causal influence of $A$ over $B$ and $\Lambda$ is a latent factor that might affect both $A$ and $B$. By introducing the instrumental variable $X$ and assuming its statistical independence from $\Lambda$, one can infer the causal strength $\beta$. For that aim, it is enough to multiply both sides of the structural equation by $x$ and compute the observed correlations, defined as $\mathrm{corr}(A,B)=\langle A,B \rangle / \langle A \rangle \langle B \rangle$ where $\langle A,B \rangle=\sum_{a,b}( a\cdot b) p(a,b)$ is the expectation value of $A$ and $B$. By doing that, we obtain that $\beta=\mathrm{corr}(X,B)/\mathrm{corr}(X,A)$.

More formally, an instrumental variable $X$ has only a direct causal influence over $A$ and should be independent of any latent factors acting as a common cause for variables $A$ and $B$. This assumption is known by various names such as the independence assumption \cite{rassen2009instrumental}, ignorable treatment assignment \cite{angrist1996identification}, no confounding for the effect of $X$ on $B$ \cite{hernan2006instruments}, and in the literature of quantum foundations is termed as the measurement independence assumption \cite{wood2015lesson,chaves2015unifying,hall2016significance,Hall2020,chaves2021causal}, an issue of crucial relevance for the violation of Bell inequalities \cite{bell1964einstein,big2018challenging}. Furthermore, even though $X$ and $B$ might be correlated, those correlations can only be mediated by $A$, the so-called exchangeability assumption \cite{lousdal2018introduction,Kedagni2020}. That is, $X$ should not have any direct causal influence over $B$. See Fig.~\ref{fig: Instrumental}a) for a directed acyclic graph (DAG) description of the instrumental scenario. Altogether, any observed distribution $p(a,b,x)$ compatible with these instrumental conditions should then be decomposable as
\begin{equation}
    p(a,b,x)=\sum_{\lambda}p(a\vert \lambda,x)p(b\vert \lambda,a)p(x)p(\lambda).
\end{equation}
Typically, instead of looking at the joint distribution $p(a,b,x)$, one rather considers the conditional distribution $p(a,b \vert x)$ that, under the same causal assumptions, can be decomposed as
\begin{equation}
\label{eq:instrumental}
    p(a,b \vert x)=\sum_{\lambda}p(a\vert \lambda,x)p(b\vert \lambda,a)p(\lambda).
\end{equation}

The set of probability distributions of the form in Eq.~(\ref{eq:instrumental}) is bounded in the space of all possible distributions $p(a,b|x)$. These bounds are given by the so-called \emph{instrumental inequalities}~\cite{pearl1995testability,bonet2013instrumentality,poderini2020exclusivity}. For the simplest case of dichotomic variables there is only one type of instrumental inequalities, which we will call Pearl's inequality~\cite{pearl1995testability} and can be summarized as follows
\begin{align}\label{eq:binary_instrumental_ineqs}
    p(j,0|0)+p(j,1|1)\leq 1,\qquad p(j,0|1)+p(j,1|0)\leq 1, \qquad \text{for } j\in \{0,1\}.
\end{align}

We have shown above that in the case of linear dependence of $B$ on $A$, the instrumental variable $X$ can be used to determine the strength of this dependence exactly. Importantly, the instrumental variable can be used for causal inference even in the absence of structural models, something typical in the context of quantum information and refereed there as the device-independent framework \cite{pironio2016focus,chaves2018quantum}. In particular, simply from the observed data $p(a,b \vert x)$ one can infer the effect of interventions on the variable $A$ and thus obtain a lower bound on the average causal effect $\mathrm{ACE}_{A \rightarrow B}$ defined as
\begin{equation}\label{eq:ACE_def}
\mathrm{ACE}_{A \rightarrow B} = \max_{a,a^{\prime},b} \vert p(b\vert \mathrm{do}(a))-p(b\vert \mathrm{do}(a^{\prime})) \vert,
\end{equation}
in which 
\begin{eqnarray}
p(b|\mathrm{do}(a)) = \sum_{\lambda}p(b|a,\lambda)p(\lambda),
\end{eqnarray}
and $\mathrm{do}(a)$ represents the intervention over the variable $A$. As shown in Ref.~\cite{Balke1997}, for the case of binary random variables $A,B$ and $X$ the value of the average causal effect in Eq.~(\ref{eq:ACE_def}) can be lower-bounded as 
\begin{equation}\label{eq: Balke 1}
    \mathrm{ACE}_{A \rightarrow B} \geq 2p(0,0|0)+p(1,1|0)+p(0,1|1)+p(1,1|1)-2.
\end{equation}
The bound above is particularly relevant because it shows that the effect of interventions can be inferred simply from the observational data. Thus, instrumental variables offer a central tool for situation where interventions are not possible. 

The bound in Eq.~(\ref{eq: Balke 1}) is one of the eight expressions given in Ref.~\cite{Balke1997} which are proven to provide non-trivial lower bounds on $\mathrm{ACE}_{A \rightarrow B}$. The three of these eight bounds can be obtained by relabeling the one in Eq.~(\ref{eq: Balke 1}) and the rest four are not interesting for our purposes, since they hold for any causal structure. For the case of more general random variables (not only binary), one can obtain a system of linear inequalities of the form
\begin{equation}
    \mathrm{ACE}_{A \rightarrow B} \geq \max_i\{C_i\},
\end{equation}
where $C_i = \sum c^{a,b,x}_ip(a,b|x)$, are linear expressions of the probabilities with $c^{a,b,x}\in \mathds{R}$ and the maximum is taken over all such expressions. These lower bounds $C_i$ can be found using the tools of linear programming \cite{boyd2004convex}. We refer to this type of bounds on the average causal effect as \emph{causal bounds}. The causal bound in Eq.~(\ref{eq: Balke 1}) we denote as $C_1$. Other bounds studied in this work are given in Section~\ref{sec: nonbinary}. 

In this work, we focus on the case where the variables $A$ and $B$ are binary, i.e., taking values $a,b=0,1$ but the instrumental variable can  take more values (we resort to an arbitrary set of values when discussing the instrumental inequalities and the following two cases $x \in\{0,1\}$ and $x \in\{0,1,2\}$, when referring to the problem of causal bounds). At the same time, the methods developed in this paper are applicable to the general case where all the random variables take values in arbitrary finite sets.

\begin{figure}
    \centering
    \includegraphics[scale = .60]{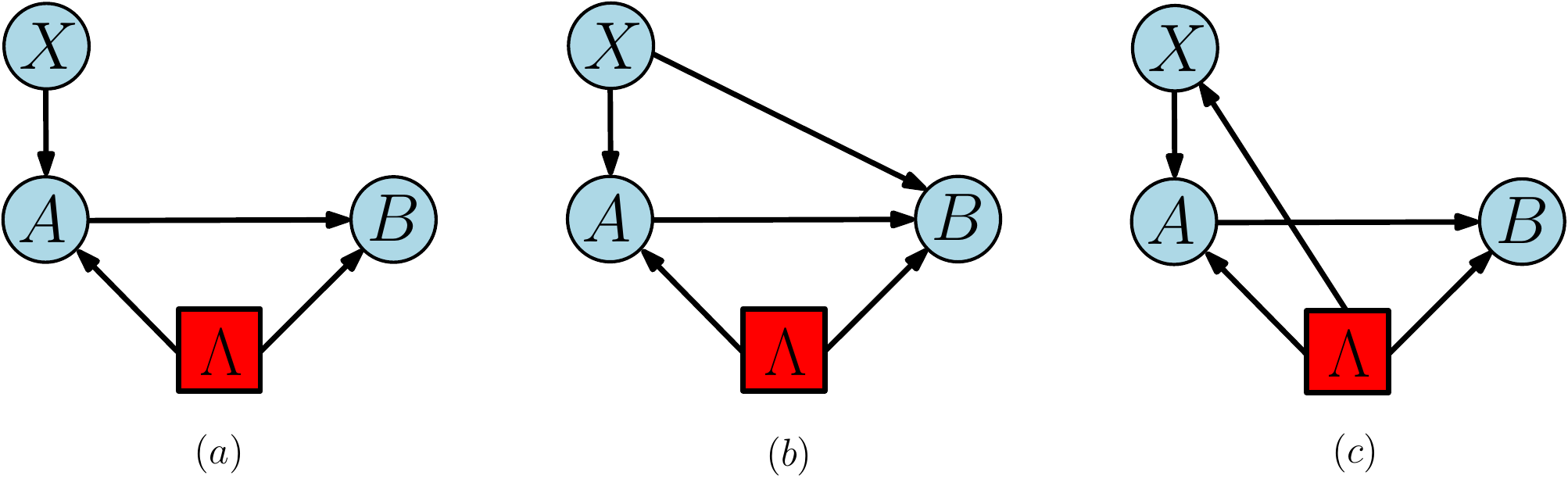}
    \caption{Causal graphs describing the instrumental scenario and its relaxations. Circular nodes correspond to observed variables, and rectangular ones are latent. Directed edges represent the causal links. (a) The instrumental scenario: the controlled variable $X$ is completely independent of a latent variable $\Lambda$. (b) A relaxed instrumental scenario, where the exchangeability assumption is relaxed, in other words, there is a direct causal influence of $X$ over $B$. We are not focusing on this relaxation.
    (c)  A relaxed instrumental scenario, where the independence assumption is relaxed. Differently from (a),  there is a causal link from $\Lambda$ to $X$. Consequently, we no longer assume that the instrumental variable $X$ and the common cause $\Lambda$ are independent, that is, $p(x,\lambda) \neq p(x)p(\lambda)$. We are focusing on this relaxation.}
    \label{fig: Instrumental}
\end{figure}

\section{Relaxing the independence assumption}
\label{sec3}
For the causal bounds such as in Eq.~\eqref{eq: Balke 1} to hold, one has to guarantee that the instrumental causal assumptions are fulfilled. If any instrumental inequality such as in Eq.~\eqref{eq:binary_instrumental_ineqs} is violated by the observed data $p(a,b\vert x)$, then one can unambiguously conclude that at least one of the instrumental assumptions does not hold. Such a violation can have two distinct roots. As shown in Refs.~\cite{chaves2018quantum,nery2018quantum,van2019quantum,agresti2020experimental}, even if one imposes the instrumental causal structure to a quantum experiment, still some instrumental inequalities can be violated. This can be seen as a stronger version of Bell's theorem \cite{bell1964einstein}, showing that correlations mediated via quantum entanglement can fail to have a description in terms of standard causal models. The second kind of mechanism, purely classical, and the one we mainly focus on in this paper, is the failure of causal assumptions.

For instance, the violation of an instrumental inequality could be motivated by a direct causal influence of $X$ over $B$, a violation of the  exchangeability assumption shown in Fig.~\ref{fig: Instrumental}b), a scenario analyzed in Ref.~\cite{chaves2018quantum}. Here, as shown in Figs.~\ref{fig: Instrumental}c) we focus on the violation of the independence assumption. Differently from the typical scenario, we no longer assume that the instrumental variable $X$ and the common source $\Lambda$ are independent, that is, $p(x,\lambda) \neq p(x)p(\lambda)$.

In order to facilitate our analysis, we focus on the DAG including an additional causal link between a latent variable  $\Lambda$ and instrument $X$ (see Fig.~\ref{fig: Instrumental}c)). We treat a realization of $\Lambda$ as a vector $(\lambda_x,\lambda_a,\lambda_b)$, where $\lambda_x\in[m_x]$, $\lambda_a$ takes its values in $[k_a^{m_x}]$, and $\lambda_b\in [k_b^{k_a}]$.  In this relaxed case, any distribution $p(a,b|x)$  factorizes as follows,
\begin{align}
\label{eq: Instrumental_md}
\nonumber
p(a,b|x) = & \frac{1}{p(x)}\sum_{\lambda_{x},\lambda_{a},\lambda_b}p(a|x,\lambda_a)p(b|a,\lambda_b)p(x|\lambda_x)p(\lambda_x,\lambda_{a},\lambda_b)\\
= & \frac{1}{p(x)}\sum_{\lambda_{a},\lambda_b}p(a|x,\lambda_a)p(b|a,\lambda_b)p(x,\lambda_{a},\lambda_b),
\end{align}
where we use the same notation $p(\cdot)$ for different response functions in order to avoid cumbersome expressions. Moreover, we took, without loss of generality, that $p(x|\lambda_x)=\delta_{x,\lambda_x}$, $\delta_{\cdot,\cdot}$ representing the Kronecker delta. Similarly, conditional probabilities $p(a|x,\lambda_a)$ and $p(b|a,\lambda_b)$ can be chosen to be deterministic, leading to,
\begin{equation}
\label{eq: Instrumental_md deterministic}
p(a,b|x) = \frac{1}{p(x)}\sum_{\lambda_{a},\lambda_b}\delta_{a,f_{\lambda_a}(x)}\delta_{b,g_{\lambda_b}(a)}p(x,\lambda_{a},\lambda_b)
\end{equation}
where $f_i(\cdot)$ and $g_i(\cdot)$ denote deterministic functions, specified by $\lambda_a$ and $\lambda_b$, respectively. 

In analogy to Ref.~\cite{chaves2015unifying}, we use a common measure of dependence between $X$ and $\Lambda$ for the instrumental scenario given by
\begin{align}\label{eq: meas dep def}
\mathcal{M}_{X:\Lambda} = & \sum_{x,\lambda_a,\lambda_b}|p(x,\lambda_a,\lambda_b) - p(x)p(\lambda_a,\lambda_b)|.
\end{align}
Crucially to our subsequent analysis, we cast it as  the  $l_1$-norm of the following vector,
\begin{eqnarray}
\label{eq: md}
\mathcal{M}_{X:\Lambda} & = & \norm{M\mathbf{q}}_{l_1},
\end{eqnarray}
where $\mathbf{q}_{\lambda_x,\lambda_a,\lambda_b}=p(\lambda_x,\lambda_a,\lambda_b)$ and for the canonical basis $\{\mathbf{e}_{i,j,k}\}_{i,j,k}$ in $\R^{m_x k_a^{m_x}k_b^{k_a}}$, we have a matrix $M$,

\begin{equation}
M = \sum_{x}\sum_{\lambda_x,\lambda_a,\lambda_b}\left(\delta_{x,\lambda_x} - p(x)\right) \mathbf{e}_{x,\lambda_a,\lambda_b}\cdot \mathbf{e}_{\lambda_x,\lambda_a,\lambda_b}^T.
\end{equation}

\subsection{Quantifying violation of the independence assumption}
The observed correlations in the instrumental experiment given by the observed probability distribution $p(a,b \vert x)$, as discussed in previous sections, allow us to evaluate the instrumental inequalities or lower bound the strength of the causal influence from $A$ to $B$. Violation of these inequalities implies that the instrumental assumptions were not met in the experiment. As mentioned before, it is important to note that this claim only works if all the latent variables are classical. Curiously, the theory of causality has recently been generalized to quantum causal modeling ~\cite{leifer2013towards,fritz2016beyond,henson2014theory,chaves2015information,pienaar2015graph,costa2016quantum,allen2017quantum}. In the latter case, the latent variables are quantum states that may be entangled, and the classical variables are obtained through quantum measurements. Quantum causal modeling differs from classical causal modeling in its predictions and as recently demonstrated in Refs.~\cite{chaves2018quantum,Gachechiladze2020,agresti2021experimental}, if the hidden common cause is allowed to be a quantum entangled state, the bounds obtained for classical instrumental causal structure can be violated. This is true for both instrumental inequalities and causal bounds.

In this paper, taking a purely classical perspective on causality, we aim to quantify how much of the above-mentioned violation translates into a relaxation of the independence assumption. More precisely, we aim to find the minimal amount of dependence necessary to explain the violation of either instrumental inequalities or causal bounds.

Given a linear inequality valid for the instrumental scenario $K_{\text{inst}} \geq 0$, (e.g.,  $K_{\text{inst}} =\mathrm{ACE}_{A\rightarrow B}- C_i$), if it is violated by a fixed amount $\alpha$, we want to establish what is the minimal amount of dependence, $\mathcal{M}_{X:\Lambda}$ that could reproduce this violation. Here, we cast this as an optimization problem,
\begin{align}
    \begin{split}
	    \min_{\mathbf{q}} \quad & \norm{M\mathbf{q}}_{l_1} \\
		\text{s.t.} \quad & K_{\text{inst}}\leq   -\alpha, \\	 
		& \sum_{\lambda_a,\lambda_b}\mathbf{q}_{x,\lambda_a,\lambda_b} = p(x), \quad \forall x\in [m_x],\\
		& \mathbf{q}\geq0 .
		  \label{Optimization: causal}
		      \end{split}
	\end{align}
Note that the normalization of $\mathbf{q}$ is implied by the normalization of $p(x)$. We are ready to state our first result.

\begin{observation}\label{obs:linear_dependency}
The minimal dependence needed to explain a fixed violation $\alpha$ of a linear inequality valid for the instrumental scenario is a monotonic convex piecewise linear function in $\alpha$.
\end{observation}

To see that this statement holds, first we bring the problem in Eq.~(\ref{Optimization: causal}) to a standard primal form of a linear program (LP)~\cite{Chaves2015}. 
\begin{align}
    \begin{split}
	    \max_{\mathbf{q, t}} \quad &- \mathbf{1}^T\cdot \mathbf{t}\\
		\text{s.t.} \quad &  \bm{[} M,-\openone \bm{]} \left[\begin{array}{c} \mathbf{q}\\ \mathbf{t} \end{array}\right]\leq \mathbf{0},\\
		&  \bm{[}-M,-\openone\bm{]}\left[\begin{array}{c} \mathbf{q}\\ \mathbf{t} \end{array}\right]\leq \mathbf{0},\\
		&K\cdot P\cdot \mathbf{q} \leq -\left[\begin{array}{c} \alpha\\ \mathbf{0} \end{array}\right],\\
	  & \Delta \cdot \mathbf{q}\leq \mathbf{p}_x,\\
		 -&\Delta\cdot \mathbf{q}\leq -\mathbf{p}_x,\\
	 -&\mathbf{q}\leq \mathbf{0}.
    \end{split}\label{eq: primal LP}
\end{align}

In the above LP, we used the following notations. $\mathbf{1}$ is the vector of $1$s, and similarly, $\mathbf{0}$ is the vector of $0$s. The matrix $K$ specifies the coefficients in the inequality $K_{\text{inst}}$ and some additional conditions that need to be specified for a specific problem  (E.g., the condition on the do-probabilities, under which $\mathrm{ACE}_{A\rightarrow B}$ becomes a linear function. If $\mathrm{ACE}_{A\rightarrow B} = |p(0|do(0))-p(0|do(1))|$, then the aforementioned condition is either $p(0|do(0)-p(0|do(1))\geq 0$ or $p(0|do(0)-p(0|do(1))\leq 0$). A matrix $P$ is a probability matrix such that its columns correspond to the deterministic assignments given by $f_{\lambda_a}(x)$ and $g_{\lambda_b}(a)$ in Eq.~\eqref{eq: Instrumental_md deterministic}. Finally, $\Delta$ denotes a matrix with entries equal to $1$ if the corresponding value of $\lambda_x$ in $\mathbf{q}$ is $x$ and $0$ otherwise for all values of $x\in [m_x]$. $\mathbf{p}_x$ is a vector of probabilities $p(x)$. 

Below, we give the corresponding dual LP to the one in Eq.~\eqref{eq: primal LP},
\begin{align}
    \begin{split}
         \min_{\mathbf{y},u,v,\mathbf{z}}  \quad & - \alpha u + \mathbf{p}_x^T \cdot \mathbf{z} \\
         \text{s.t.} \quad &  M^T\cdot\mathbf{y}+P^T\cdot K^T \left[\begin{array}{c} {u}\\ \mathbf{v} \end{array}\right]  +\Delta^T\cdot  \mathbf{z} \geq \mathbf{0},\\
         & \mathbf 0\leq \mathbf y \leq \mathbf 2,\\
         &  u \geq 0,\quad  \mathbf{v} \geq 0.
    \end{split}\label{eq: MD dual LP}
\end{align}
In the above, we introduced the notation $\mathbf{2}$, which is a vector of all $2$s and $\mathbf{y,z,v}$ and $u$ are the dual variables.

We can see from the above dual formulation of the LP that the solution must be piecewise linear in $\alpha$. Indeed, since the feasibility region of the above LP is a polytope defined by a finite set of constraints, there is a finite set of possibly optimal assignments to $u$ and $\mathbf{z}$. Hence, if we change $\alpha$ slowly from $0$ to its maximal value, the solution for $u$ might change in at most a finite number of points for $\alpha$. Moreover, it must be clear that for $\alpha=0$, i.e., in case the inequality $K_{\mathrm{inst}}\geq 0$ is valid, no dependence is required, and thus the output of the optimization problem should be $0$. Thus, it must also hold that $\mathbf{p}_x^T \cdot \mathbf{z}=0$ in the vicinity of $\alpha=0$. 

Since any solution of the above LP, defining the slope $u$, remains a solution for all valid values of $\alpha$, it follows that even though the whole function can be piecewise linear, i.e., have different slopes, these slopes may only increase. In other words, the resulting dependence of $\mathcal{M}_{X:\Lambda}$ on $\alpha$ is convex and monotonic.

Finally, we must note that the primal problem is feasible, if the violation $\alpha$ is at most the maximum possible, which can be attained by one of the deterministic assignments given by $f_{\lambda_a}(x)$ and $g_{\lambda_b}(a)$ in Eq.~\eqref{eq: Instrumental_md deterministic} expressed as columns of matrix $P$.

\subsection{Dependencies in the simplest instrumental scenario}

Building on the results of this section, here we investigate the minimal required dependence for a fixed violation of instrumental inequalities and bounds on ACE in the simplest instrumental scenario when all the observed random variables are binary. For the both types of inequalities, namely instrumental inequality in Eq.~(\ref{eq:binary_instrumental_ineqs}) and causal bound in Eq.~(\ref{eq: Balke 1}) we give exact solutions to the corresponding linear programs in Eq.~(\ref{eq: MD dual LP}).

\begin{lemma}\label{lemma:instr_ineq}
For the instrumental scenario with binary observed random variables $X,A,B$ and a latent variable $\Lambda$, the minimal dependence required to explain a violation of the instrumental inequality by $\alpha$ is $\mathcal{M}_{X:\Lambda} = 4p(X=0)p(X=1)\alpha$.
\end{lemma}
\begin{proof}
All the binary instrumental inequalities are given in Eq.~(\ref{eq:binary_instrumental_ineqs}). We choose one of them (the results work for any other choice too, due to symmetry present in the problem) and insert it into the primal problem, 
\begin{equation}\label{eq:Instrumental_binary}
    K_{\text{inst}}=-p(00|0)-p(01|1)+1.
\end{equation} 
In the dual LP in Eq.~(\ref{eq: MD dual LP}), the matrix $K$ is $1\times 8$, which is a matrix representation of the expression $K_{\text{inst}}$ above. The matrix $P$ is $8\times 32$ with each column corresponding to a deterministic assignment of $X,A$ and $B$ given $\lambda$. The vector $\mathbf{z}$ has two components, which we call $z_0$ and $z_1$ and the vector $\mathbf{y}=[y_0,y_1,\dots,y_{31}]$ is $32$-dimensional. Finally, there is no vector $\mathbf{v}$ in our LP, as there are no additional linear constraints in $K$. 

From the definition of $M$, we derive that $M^T \mathbf{y}=\left[\begin{array}{c} {\ \ p(X=1)  \mathbf{\tilde{y}}}\\ {-p(X=0)  \mathbf{\tilde{y}}} \end{array}\right]$, where $-\mathbf{2}\leq\mathbf{\tilde{y}}\leq \mathbf{2}$ is a column vector, $\mathbf{\tilde{y}}^T=[\tilde{y}_0\dots\tilde{y}_{15}]$, where $\tilde{y}_i = y_i-y_{i+16},\; i\in [16]$. Moreover, note that $\Delta^T \mathbf{z}= \left[\begin{array}{c} {z_0  \mathbf{1}}\\ {z_1  \mathbf{1}} \end{array}\right]$.  Taking all the above into account, the LP takes the following form,
\begin{align}\label{eq:dual_linear_program_instr}
    \begin{split}
         \min_{\mathbf{\tilde y},u,z_0,z_1}  \quad & - \alpha u +p(X=0)z_0+p(X=1)z_1 \\
         \text{s.t.} \quad &  p(X=1)\tilde y_i+[P^T\cdot K^T]_i \,u  +z_0 \geq 0, \quad \forall i\in [16],\\
         \quad   - & p(X=0)\tilde y_i+[P^T\cdot K^T]_{i+16}\, u  +z_1 \geq 0, \quad \forall i\in [16],\\
         - &\mathbf 2 \leq \mathbf {\tilde{y}} \leq \mathbf 2, \quad u \geq 0.
    \end{split}
\end{align}

Here $[P^T\cdot K^T]_i$ is the $i$-th term of the vector $P^T\cdot K^T$. For $i [16]$, the expression $[P^T\cdot K^T]_i$ can take one of the two possible values, either $\left(1-\frac{1}{p(X=0)}\right)$ or $1$, and for $i \in \{16,\dots,31\}$, it can take one of the two possible values $\left(1-\frac{1}{p(X=1)}\right)$ or $1$. This simplifies the problem and by erasing redundant constraints we arrive at the final form of the LP which we solve explicitly.
\begin{align}\label{eq: lemma 1 final LP}
    \begin{split}
         \min_{\tilde{y}_0,\tilde{y}_1,\tilde{y}_2,u,z_0,z_1}  \quad & - \alpha u+ p(X=0)z_0+p(X=1)z_1 \\
         \text{s.t.} \quad &  p(X=1)\tilde y_0 - \frac{p(X=1)}{p(X=0)}\,u  +z_0 \geq 0,\quad p(X=1)\tilde y_i+u+z_0 \geq 0,\quad i\in\{1,2\},\\
         \quad   - & p(X=1)\tilde y_1-u  +\frac{P(X=1)}{P(X=0)} z_1\geq 0, \quad -p(X=1)\tilde y_i + \frac{p(X=1)}{p(X=0)}\, u  +\frac{p(X=1)}{p(X=0)}z_1 \geq 0, \  i\in\{0,2\},\\
         - &2 \leq \tilde{y}_i \leq 2,\quad i\in \{0,1,2\}, \qquad u \geq 0.
    \end{split}
\end{align}
By summing the first and the last inequalities for $i=0$, we directly get, $p(X=0)z_0+p(X=1)z_1\geq 0$. Finally, summing up the two inequalities, where the variable $u$ has a negative coefficient, we obtain an upper-bound on $u$,
\begin{align}
    u\leq & (\tilde y_0-\tilde y_1) p(X=0) p(X=1)+p(X=0)z_0+p(X=1)z_1\\
    \leq & 4 p(X=0) p(X=1)+p(X=0)z_0+p(X=1)z_1.
\end{align}
Using the upper-bound on $u$ we get that the objective function can be lower-bounded by the expression
\begin{align}
    \begin{split}
       -4 \alpha p(X=0) p(X=1)+\left(p(X=0)z_0+p(X=1)z_1\right)(1-\alpha)\geq   -4 \alpha p(X=0)P(X=1).
    \end{split}
\end{align}
As the final step, we note that the assignment: $u = 4p(X=0) p(X=1)$, $\tilde y_0=2$, $\tilde y_1=-2$, $\tilde{y}_2=0$, $z_0= 2p(X=1)(1-2p(X=0))$ and $z_1=-\frac{p(X=0)}{p(X=1)}z_0$ is a feasible point of the LP. Thus, $\mathcal{M}_{X:\Lambda} = 4p(X=0)p(X=1)\alpha$.
\end{proof}

We conclude that for a given violation $\alpha$, the uniformly distributed instrumental variable $X$ requires the highest dependence. The reverse also holds true: if the  instrumental variable is uniformly random, a given dependence will permit the lowest amount of violation. Our result implies that even though we do not have direct empirical access to the common source between $A$ and $B$, from observational data $p(a,b\vert x)$ alone we can lower-bound the amount of dependence $\mathcal{M}_{X:\Lambda}$ present in a given experiment.

Next we investigate how the violation of the lower bound on ACE as in Eq.~(\ref{eq: Balke 1}) translates to the required measurement dependence. 
\begin{lemma}\label{lemm: ace meas dep}
For the instrumental scenario with binary observed random variables $X,A,B$ and a latent variable $\Lambda$, the minimal measurement dependence required to explain a violation of the lower bound on ACE as in Eq.~(\ref{eq: Balke 1}) by $\alpha$ is $\mathcal{M}_{X:\Lambda} = \frac{4p(X=0)p(X=1)}{2-p(X=0)}\alpha$.
\end{lemma}
\begin{proof}
The proof has a similar structure as the previous one, however it is more involving. The main reason for this is that the expression $K_{\text{inst}}=\mathrm{ACE}_{A\rightarrow B}-C_1$ is written not only in terms of probabilities $p(a,b|x)$, but also in terms of do-probabilities. Additionally, by definition ACE is not linear in do-probabilities, but we can linearize it without loss of generality by requesting that $p(0|\mathrm{do}(0))-p(0|\mathrm{do}(1))\geq 0.$
 In the dual LP in Eq.~(\ref{eq: MD dual LP}), the matrix $K$ is then $2\times 12$, which is a matrix representation of the expression $K_{\text{inst}}=\mathrm{ACE}_{A\rightarrow B}-C_1$ . The matrix $P$ is $12\times 32$ with each column corresponding to a deterministic assignment of $X,A$ and $B$ given $\lambda$ (which also gives deterministic assignments to the do-probabilities). The vector $\mathbf{z}$ has two components, which we call $z_0$ and $z_1$ and the vector $\mathbf{y}=[y_0,y_1,\dots,y_{31}]$ is $32$-dimensional. Finally, there is only a single element in vector $\mathbf{v}$ in our LP, which corresponds to the positivity of $p(0|\mathrm{do}(0))-p(0|\mathrm{do}(1))\geq 0$. 

The matrix $M$ is the same as in Lemma~\ref{lemma:instr_ineq}, $M^T \mathbf{y}=\left[\begin{array}{c} {\ \ p(X=1)  \mathbf{\tilde{y}}}\\ {-p(X=0)  \mathbf{\tilde{y}}} \end{array}\right]$, where $-\mathbf{2}\leq\mathbf{\tilde{y}}\leq \mathbf{2}$ is a column vector, $\mathbf{\tilde{y}}^T=[\tilde{y}_0\dots\tilde{y}_{15}]$, where $\tilde{y}_i = y_i-y_{i+16},\; i\in [16]$ and  $\Delta^T \mathbf{z}= \left[\begin{array}{c} {z_0  \mathbf{1}}\\ {z_1  \mathbf{1}} \end{array}\right]$. We need to solve the following LP,
\begin{align}
    \begin{split}
         \min_{\mathbf{\tilde y},u,v,z_0,z_1}  \quad & - \alpha u +p(X=0)z_0+p(X=1)z_1 \\
         \text{s.t.} \quad &  p(X=1)\tilde y_i+[P^T\cdot K^T]_{i,0} \,u+[P^T\cdot K^T]_{i,1} \,v + z_0 \geq 0, \quad \forall i\in [16],\\
         \quad   - & p(X=1)\tilde y_i+\frac{p(X=1)}{p(X=0)}\left([P^T\cdot K^T]_{i+16,0}\, u+[P^T\cdot K^T]_{i+16,1}\, v\right) +\frac{p(X=1)}{p(X=0)}z_1 \geq 0, \quad \forall i\in [16],\\
         - &\mathbf 2 \leq \mathbf {\tilde{y}} \leq \mathbf 2, \quad u \geq 0, \quad v \geq 0,
    \end{split}
\end{align}
where we denoted by $[P^T\cdot K^T]_{i,j}$ the element of the matrix $P^T\cdot K^T$ on $i$-th row and $j$-th column (counting from $0$). We give rows of $[K\cdot P]$ here for completeness: $ [K\cdot P]_{0} = \Big[\frac{2p(X=0)-2}{p(X=0)}$ , $\frac{3p(X=0)-2}{p(X=0)} , 1 , 2 , \frac{2p(X=0)-2}{p(X=0)} , \frac{3p(X=0)-2}{p(X=0)},$ $1,$  $2,$  $2,$  $\frac{3p(X=0)-1}{p(X=0)} , 1 , \frac{2p(X=0)-1}{p(X=0)} , 2, $  $ \frac{3p(X=0)-1}{p(X=0)} , 1 , \frac{2p(X=0)-1}{p(X=0)},2,1,\frac{p(X=1)-1}{p(X=1)},\frac{2p(X=1)-1}{p(X=1)},2,\frac{3p(X=1)-1}{p(X=1)},1,\frac{2p(X=1)-1}{p(X=1)},2,3,$ $\frac{p(X=1)-1}{p(X=1)},$ $\frac{2p(X=1)-1}{p(X=1)},$ $2,$ $\frac{3p(X=1)-1}{p(X=1)},$ $1,\frac{2p(X=1)-1}{p(X=1)}\Big]$, $[K\cdot P]_{1} = [0, -1, 1, 0, 0, -1, 1, 0, 0, -1, 1, 0, 0, -1, 1, 0, 0, -1, 1, 0, 0,$ $-1,$  $1,$  $0,$  $0,$  $-1,$ $1,$  $0,$  $0,$ $-1,$  $1, 0]$.

First, we derive an upper-bound on $u$. For the feasibility region the following must hold true for any $i,j\in [16]$ (which one gets simply by summing the two types of constraints above), 
\begin{align}
\begin{split}
    p(X=1)(\tilde y_i-\tilde{y_j})&+u\left([P^T\cdot K^T]_{i,0}+\frac{p(X=1)}{p(X=0)}[P^T\cdot K^T]_{j+16,0} \right)\\
    +z_0+\frac{p(X=1)}{p(X=0)}z_1 &+v\left([P^T\cdot K^T]_{i,1}
    +\frac{p(X=1)}{p(X=0)}[P^T\cdot K^T]_{j+16,1} \right)\geq 0.
    \end{split}
\end{align}
For $i=5$ and $j=5$, the values $[P^T\cdot K^T]_{5,0}=\frac{3p(X=0)-2}{p(X=0)}$, $[P^T\cdot K^T]_{21,0}=\frac{3p(X=1)-1}{p(X=1)}$, $[P^T\cdot K^T]_{5,1}=-1$, and $[P^T\cdot K^T]_{21,1}=-1$ lead to the condition $p(X=0)z_0+p(X=1)z_1\geq v$. For $i=0$ and $j=2$, for which $[P^T\cdot K^T]_{0,0} = 2-\frac{2}{p(X=0)}$ and $[P^T\cdot K^T]_{18,0} = 1-\frac{1}{p(X=1)}$,  $[P^T\cdot K^T]_{0,1}=0$, $[P^T\cdot K^T]_{18,1}=1$, we get
\begin{equation}
    (\tilde{y}_0-\tilde{y}_2)p(X=1)+ u \frac{p(X=0)-2}{p(X=0)}+\frac{P(X=1)}{P(X=0)}v +z_0+\frac{p(X=1)}{p(X=0)}z_1\geq 0,
\end{equation}
which means that $u\leq \frac{1}{2-p(X=0)}\left(4p(X=0)p(X=1)+p(X=1)v+p(X=0)z_0+p(X=1)z_1\right)$, since  $\tilde{y}_0-\tilde{y}_2\leq 4$.

Inserting this value in the objective function, we get,
\begin{align}
   & \frac{-\alpha}{2-p(X=0)}\left(4p(X=0)p(X=1)+p(X=1)v+p(X=0)z_0+p(X=1)z_1\right)+(p(X=0)z_0+p(X=1)z_1)\\
 \geq &   \frac{-\alpha 4p(X=0)p(X=1)}{2-p(X=0)}+\left(p(X=0)z_0+p(X=1)z_1\right)\left(1-\frac{\alpha}{2-p(X=0)}\right)-\frac{\alpha p(X=1)v}{2-p(X=0)}\\
 \geq &   \frac{-\alpha 4p(X=0)p(X=1)}{2-p(X=0)}+v\left(1-\alpha \right)\geq \frac{-\alpha 4p(X=0)p(X=1)}{2-p(X=0)}.
\end{align}
The last step follows as $\alpha\leq 1$.
As the final step, we note that the assignment: $u=\frac{4p(X=0)p(X=1)}{2-p(X=0)}$, $v=0$, $\tilde y_0=\tilde y_1 = \tilde y_4 = \tilde y_8 = \tilde y_9=\tilde y_{12} = 2$, $\tilde y_2 = \tilde y_3 = \tilde y_6 = \tilde y_7=\tilde y_{10}= \tilde y_{14} = -2$, $\tilde y_5 = \tilde y_{13} = 2-\frac{4p(X=0)}{2-p(X=0)}$, $\tilde y_{11} = \tilde y_{15} =-\frac{2p(X=0)}{2-p(X=0)}$, and $z_0= -2p(X=1)(1-\frac{4P(X=1)}{2-p(X=0)})$, $z_1=-\frac{p(X=0)}{p(X=1)}z_0$ is a feasible point of the LP, which means that the lower bound of $-\alpha\frac{4p(X=0)p(X=1)}{2-p(X=0)}$ on the objective function is achievable.
\end{proof}

Until now we asked a question which degree of measurement dependence is required to explain violation of a linear inequality (e.g., instrumental inequalities or causal bounds) and we gave an analytical solution for the simplest scenario with binary observed variables. 
One can, however, ask the reverse question of how the linear inequalities change in the simplest instrumental scenario, if some level of measurement dependence is present in a given setup. This is the \emph{inverse} problem to the one considered in this section. Since both of these problems aim at estimating the same dependency, they have the same solution, namely the piecewise linear dependence in Observation~\ref{obs:linear_dependency}. 
As a result, we can derive adapted linear inequalities (e.g., binary instrumental inequalities and causal bounds) that accounts for the dependence between $X$ and $\Lambda$, explicitly.

\begin{corollary}\label{lemma:mod_inequalities}
Given a linear inequality valid for the simplest instrumental scenario, $K_{\text{inst}} \geq 0$, the \textit{adapted} linear inequality in terms of the measurement dependency is,
\begin{equation}
    K_{\text{inst}}+\frac{\mathcal{M}_{X:\Lambda}}{u}\geq 0,
\end{equation}
where $u$ is the optimization parameter of the dual LP in Eq.~(\ref{eq: MD dual LP}).

\end{corollary}
The above corollary shows that one can still infer cause and effect relations even with non-perfect instruments.  Also note that in case of independence, $\mathcal{M}_{X:\Lambda}=0$, we directly recover the inequalities valid in instrumental scenario (Pearl's inequality in Eq.~\eqref{eq:binary_instrumental_ineqs} and the causal bound in Eq.~\eqref{eq: Balke 1}). 

For a more general case, when the instrumental variable can take more than two values, adapting a linear inequality valid for the perfect instrumental scenario is also possible. However, it is a more involving task as the minimal measurement dependence does not have to be linear in the observed violation, as pointed out in Observation~\ref{obs:linear_dependency}. We give numerical treatment for this problem in Section~\ref{sec: nonbinary} and in Fig.~\ref{fig: Beyond_binary}.

\subsection{Informational cost}
Above we used the $l_1$-norm (see Eq.~(\ref{eq: meas dep def})) to quantify the level of dependence in the instrumental scenario. Another common measure used to quantify the dependence between two random variables is the \emph{information cost} \cite{Hall2020,chaves2021causal}, given by the Shannon mutual information, a measure of particular relevance in the entropic approach to causal inference \cite{fritz2012entropic,chaves2014inferring,budroni2016indistinguishability}. In this case, we are interested in quantifying $I(X;\Lambda) = H(X)-H(X|\Lambda)$, where $H(X)=-\sum_{x}p(x)\log{p(x)}$ is the Shannon entropy of $X$ and $H(X|\Lambda)$ is the conditional Shannon entropy of $X$ given $\Lambda$, respectively, and logarithm is taken to be base $2$.  In particular, we ask a question of the minimal required information cost $I(X;\Lambda)$ that would allow for a violation of instrumental inequality in Eq.~(\ref{eq:binary_instrumental_ineqs}). For convenience, let us again use the notation
\begin{equation}
    K_{\mathrm{inst}} = -p(0,0|0)-p(0,1|1)+1.
\end{equation}
If no dependence between $X$ and $\Lambda$ is allowed, then $K_{\mathrm{inst}}\geq 0$. We are now ready to present our next result.
\begin{lemma}\label{lemma:inf_cost_binary}
For the instrumental scenario with binary observed random variables $X,A,B$ and a latent variable $\Lambda$, with $X$ uniformly distributed, the minimal informational cost required to explain a value $K_{\mathrm{inst}} < 0$ of instrumental inequality is $I(X;\Lambda) = 1-h\left(\frac{1-K_{\mathrm{inst}}}{2}\right)$, where $h(p) = -p\log(p)-(1-p)\log(1-p)$ is the binary entropy.
\end{lemma}

\begin{proof} We rewrite the conditional join probabilities occurring in the expression $K_{\mathrm{inst}}$ using decomposition in  Eq.~(\ref{eq: Instrumental_md}) and the following notations for the deterministic assignments  $\Lambda^{(0)}_b=\{\lambda_b\ |\ p(B=0|X=0, \lambda_b) =1 \}$  and $\Lambda^{(1)}_b=\{\lambda_b\ |\  p(B=1|X=1, \lambda_b) =1 \}$.

\begin{align}\begin{split}
    \frac{1-K_{\mathrm{inst}}}{2} & =     \sum_{\lambda_a}\sum_{\lambda_b\in \Lambda_b^{(0)}} p(A=0|X=0, \lambda_a)p(X=0,\lambda_a, \lambda_b)+\sum_{\lambda_a}\sum_{\lambda_b\in \Lambda_b^{(1)}} p(A=0|X=1, \lambda_a)p(X=1,\lambda_a, \lambda_b)\\
   &  \leq  \sum_{\lambda_a}\sum_{\lambda_b\in \Lambda_b^{(0)}} p(X=0,\lambda_a, \lambda_b)+\sum_{\lambda_a}\sum_{\lambda_b\in \Lambda_b^{(1)}} p(X=1,\lambda_a, \lambda_b) \\ &=\sum_{ \lambda_b\in \Lambda_b^{(0)}} p(X=0, \lambda_b)+\sum_{ \lambda_b\in \Lambda_b^{(1)}} p(X=1, \lambda_b)\\
    & =  p(X=0,\lambda_b\in  \Lambda_b^{(0)})+p(X=1,\lambda_b\in  \Lambda_b^{(1)})= p(X=E), 
\end{split}\label{eq:inf_cost_binary_proof}
\end{align}
where $E$ is a random variable such that $E=0$ if $\lambda_b \in \Lambda_b^{(0)}$,  and $E=1$ if $\lambda_b \in \Lambda_b^{(1)}$. Since $E$ concerns a particular grouping of latent variable $\Lambda$, we can first use the data processing inequality and then Fano's inequality to obtain,
\begin{align}
    I(X;\Lambda)\geq  I(X;E)= H(X)-H(X|E)\geq 1- h(X=E) \geq 1-h\left(\frac{1-K_{\mathrm{inst}}}{2}\right).
\end{align}
The last inequality follows since we are only interested in the cases when $K_{\mathrm{inst}} < 0$. The above lower bound is tight for all $K_{\mathrm{inst}} < 0$, since we can always set the following assignments:  $p(A=0|X=0, \lambda_a) = p(A=0|X=1, \lambda_a) = 1$,  and  $p(X=0|\lambda_a,\lambda_b)=\frac{1-K_{\mathrm{inst}}}{2}$,  $p(X=1|\lambda_a,\lambda_b)=1-\frac{1-K_{\mathrm{inst}}}{2}$, $\forall \lambda_a,\lambda_b \in \Lambda$.
\end{proof}
The same result applies to any of the four instrumental inequalities in Eq.~(\ref{eq:binary_instrumental_ineqs}).

\subsection{Beyond the binary case}\label{sec: nonbinary}
So far we have restricted our attention to the case where all variables are binary. Here, we generalize the results for the instrumental variable, which can take more values.  

Concerning instrumental inequalities, if the variables $A$ and $B$ are binary, it is known that the instrumental scenario is completely characterized by three inequalities up to the relabelings of the variables, $I_i\leq 0$,  $i= 1,2,3$~\cite{Kedagni2020}. The inequality $I_1\leq 0$ corresponds to Pearl's inequality and was already discussed in the binary case (See Eq.~(\ref{eq:binary_instrumental_ineqs})), the second one is known as Bonet's inequality~\cite{bonet2013instrumentality},
\begin{equation}\label{eq:Bonet}
p(0,1|0)-p(0,1|1)-p(1,1|1)-p(1,0|2)-p(0,1|2)\leq 0,    
\end{equation}
and the third one is Kedagni's inequality~\cite{Kedagni2020},
\begin{equation}\label{eq:Kegani}
    p(0,0|0)+p(1,0|0)-p(0,1|1)-p(1,0|1)-p(0,0|2)-p(1,0|2)-p(0,0|3)-p(1,1|3)\leq 0.
\end{equation}
One can obtain other inequalities from Refs.~\cite{bonet2013instrumentality,Kedagni2020} by relabeling inputs and outputs and by coarse graining values of $X$.

Considering the case where $X$  assumes up to three possible values, we obtained two new classes of causal bounds, for which we give two representatives below. All the other causal bounds for three inputs can be obtained by relabeling inputs or outputs in these two inequalities.
\begin{eqnarray}
\label{eq:ace34}
     C_2 & = & p(0,0|0) + p(0,0|2) + p(1,0|0) + p(1,1|1) + p(1,1|2) - 2.\\
     C_3 & = & p(0,0|0)+p(0,0|1)-p(0,1|1)+p(0,1|2)+p(1,0|0)-p(1,0|1)+p(1,1|1)+p(1,1|2)-2.
\end{eqnarray}

For all the causal bounds and the instrumental inequalities we use the LP in Eq.~(\ref{Optimization: causal}) to estimate the  minimal measurement dependency in order to explain the violation by the amount of $\alpha$. The results are summarized in Fig.~\ref{fig: Beyond_binary}.  

Even though we only provide closed formula solutions of the LPs in the simplest binary case, in more general scenarios, for a given distribution $p(x)$, it is sufficient to solve the LP in a very few points due to the nature of the functional dependence being convex piecewise linear. For example, for the instrumental inequality $I_2\leq 0$, the numerical results in Fig.~\ref{fig: Beyond_binary} suggest that for the chosen fixed distributions of $X$, the minimal measurement dependence $\mathcal{M}_{X:\Lambda}$ is linear in $\alpha$. We could, however, reach the same conclusion by solving the LP for two different values of $\alpha$ in the interval $\alpha\in(0,1]$ for the same fixed distributions of $X$. The first value of $\alpha$ can be arbitrary, but the second one must be equal to $\alpha=1$. Additionally, we know that for $\alpha=0$, the measurement dependence  $\mathcal{M}_{X:\Lambda}=0$. If the  values of the minimal measurement dependence corresponding to these three points belong to the same straight line, we invoke the convexity property, and conclude that, $\mathcal{M}_{X:\Lambda}=u\alpha$, where $u$ is the slope of the obtained straight line. For example, for the uniformly distributed $X$,  $\mathcal{M}_{X:\Lambda}=\frac{2}{3}\alpha$, where the coefficient of $\frac{2}{3}$ can be obtained from the LP up to the numerical precision.

\begin{figure}
    \centering
    {\includegraphics[width = .33\textwidth]{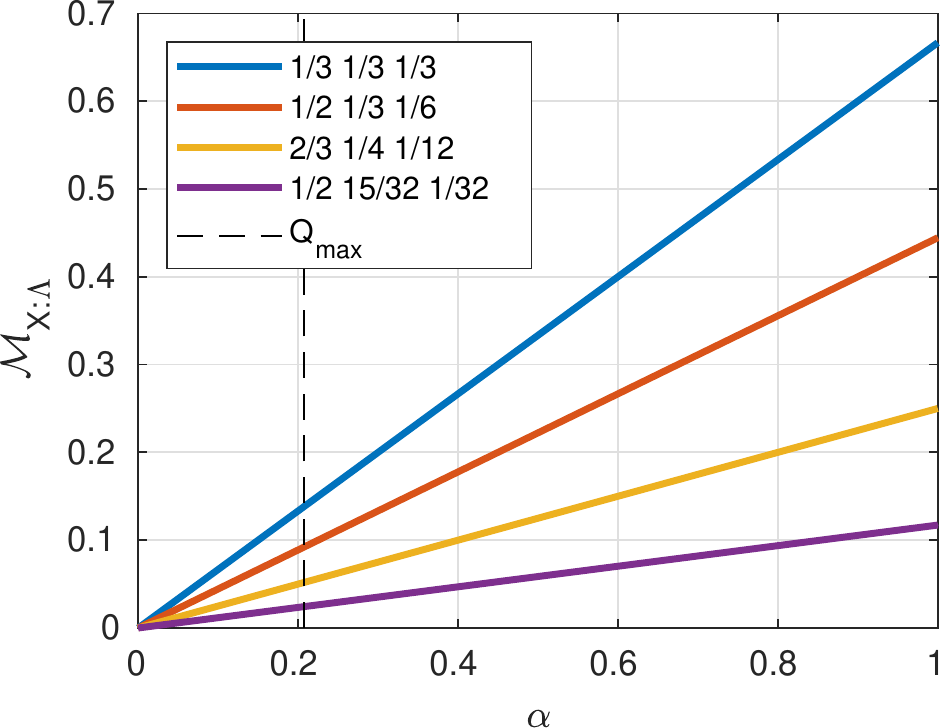}}
       {\includegraphics[width = .33\textwidth]{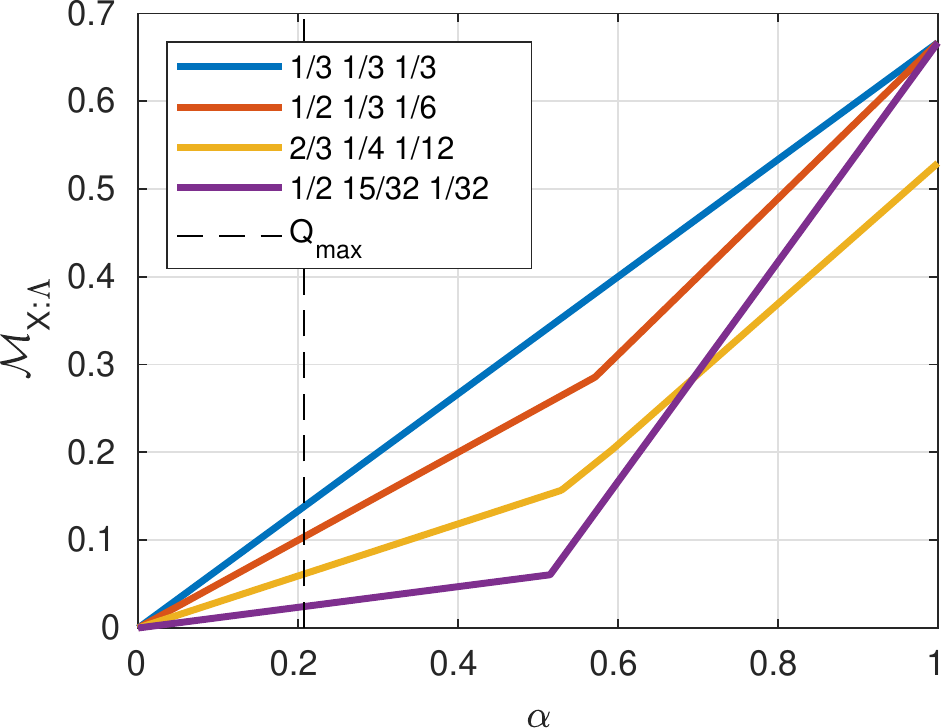}}
     {\includegraphics[width = .33\textwidth]{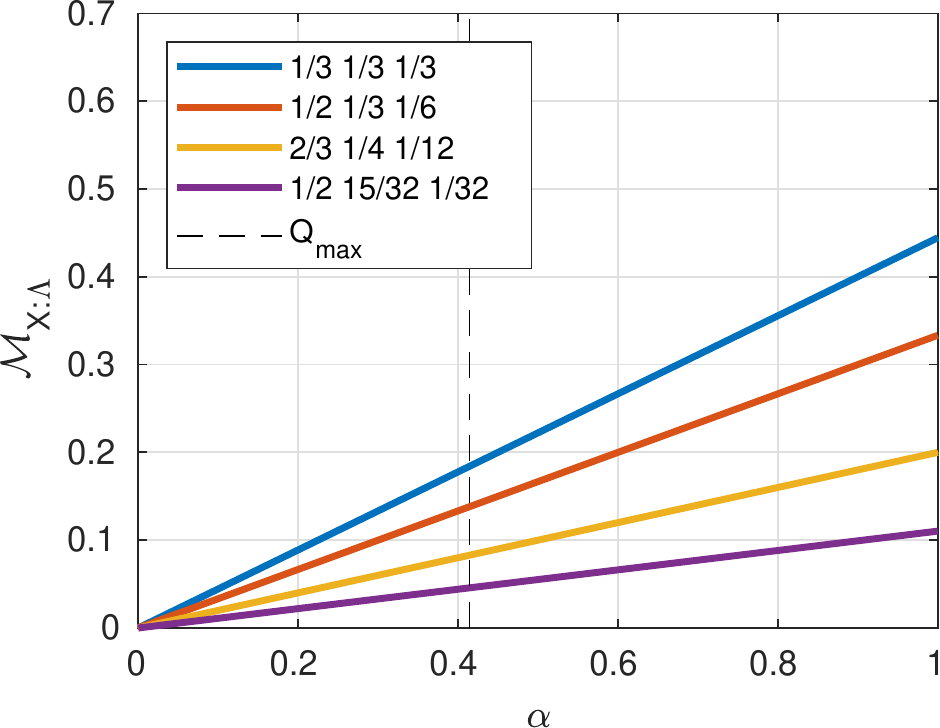}}
    \caption{Measurement dependence $\mathcal{M}_{X:\Lambda}$ for violations $\alpha$ of instrumental inequality $I_2$ (left), and causal bounds $C_2$ (center) and $C_3$ (right). The maximal violation of each inequality attainable in quantum theory is marked by $Q_{\mathrm{max}}$ (See Section~\ref{sec4}).}
    \label{fig: Beyond_binary}
\end{figure}

\section{Quantum violations of instrumental tests}
\label{sec4}

We saw that the instrumental and causal bounds can be violated if a certain amount of measurement dependence is present between an instrumental variable and a classical common cause $\Lambda$. However, a violation is also possible if we do not assume any relaxation on the instrumental scenario, but instead we consider the case, when the unobserved common cause can be a quantum state~\cite{chaves2018quantum, Gachechiladze2020, agresti2021experimental}. More precisely, in the quantum instrumental scenario considered here, all three observable variables, $X$, $A$ and $B$, are still classical random variables,  but instead of the latent variable $\Lambda$, we have a latent quantum state $\rho_{AB}$. This type of quantum causal model produces observable correlations using the \textit{Born rule} for measurements in quantum mechanics,
\begin{equation}
    p_Q(a,b|x)=\tr\left[(M^x_a\otimes N^a_b )\rho_{AB}\right].
\end{equation}
Here $\rho_{AB}$ is a quantum state of two subsystems, represented by the so-called \textit{density matrix} that is a positive, trace-$1$ linear operator acting on the tensor product of two Hilbert spaces $\mathcal{H}_A\otimes \mathcal{H}_B$,  $M^x_a$ is a positive operator acting on the first subsystem (Hilbert space $\mathcal{H}_A$) and describes a measurement depending on the choice $x$ with outcome $a$. Similarly,  $N^a_b$  is a positive operator acting on the second subsystem  (Hilbert space $\mathcal{H}_B$) and describes a measurement depending on the choice $a$ (which is the measurement outcome obtained on the first subsystem) with outcome $b$. 

In the case of the simplest instrumental scenario, the statistics obtained from a latent quantum state cannot violate the instrumental inequalities in Eq.~(\ref{eq:binary_instrumental_ineqs}). Remember, that such inequalities can be violated if the measurement dependence is present. However, it was shown in Refs.~\cite{chaves2018quantum,Gachechiladze2020} that in the case of  binary variables, the causal bound in Eq.~(\ref{eq: Balke 1}) can be violated without any measurement dependence if the intervention on $A$ is made in the quantum instrumental scenario. In a full analogy with the classical case, one can define quantum interventions as
\begin{equation}
    p_Q(b|\mathrm{do}(a))=\tr\left[\left(\mathbbm{1}\otimes N^a_b \right)\rho_{AB}\right], 
\end{equation}
where a measurement is performed only on the second subsystem. This implies that if an actual intervention is made, the observed
quantum average causal effect (qACE) is given by,
\begin{equation}
    \mathrm{qACE}_{A \rightarrow B}=\max_{a,a',b}\left|\tr\left[\left(\mathbbm{1}\otimes (N^a_b-N^{a'}_b) \right)\rho_{AB}\right]\right|.
\end{equation}

Ref.~\cite{Gachechiladze2020} showed that any pure entangled quantum state $\rho_{AB}=\ket{\psi}\bra{\psi}$,  where $\ket{\psi}=\sin\alpha\ket{0}\otimes \ket{0} +\cos{\alpha} \ket{1} \otimes \ket{1}\in \mathbbm{C}^2\otimes \mathbbm{C}^2$ and appropriate incompatible quantum measurements, $M_a^x=\frac{1}{2}\left(\mathbbm{1}+(-1)^a(\sin\theta_x\sigma_X+\cos\theta_x\sigma_Z)\right)$ and $N_b^a=\frac{1}{2}\left(\mathbbm{1}+(-1)^b(\sin\eta_a\sigma_X+\cos\eta_a\sigma_Z)\right)$, where $\sigma_X$ and $\sigma_Z$ are Pauli matrices and the vectors $\ket{0}$ and $\ket{1}$ are normalized  eigenstates of $\sigma_Z$, can violate the bound in Eq.~(\ref{eq: Balke 1}). The maximal possible violation was numerically obtained (and was verified by the hierarchy of semidefinite programs~\cite{navascues2007bounding}) to be  $3-2\sqrt2\approx 0.1716$.  

Using the results of the previous sections, we can conclude that in order to explain such a quantum violation, the amount of minimum measurement dependency in the classical instrumental causal structure must at least be $\mathcal{M}_{X:\Lambda}=\frac{4p(X=0)p(X=1)}{2-p(X=0)}(3-2\sqrt{2})$ and is maximal for $P(X=0)=(2-\sqrt{2})$ and is equal to $\mathcal{M}_{X:\Lambda}=(68 - 48 \sqrt{2})\approx0.1177 $.

In case of more general instrumental scenario, where $X$ can take more than two values, $I_2\leq 0$ and $I_3\leq 0$ can be violated by quantum states and measurements~\cite{chaves2018quantum}, both by the maximally entangled state (that is when $\sin{\alpha}=\frac{1}{\sqrt{2}}$) with the amount of $\left(\frac{1}{\sqrt{2}}-\frac12 \right)\approx0.2071$ and $ \sqrt2-1\approx 0.4142$, respectively. See Fig.~\ref{fig: Beyond_binary}~(left) for the relation between the minimal required measurement dependence in classical instrumental scenario and the amount of violation of $I_2\leq 0$ for various probability distributions of the instrumental variable. In particular, for the uniformly distributed instrumental variable, the minimal measurement dependence required to explain the quantum violation is $\mathcal{M}_{X:\Lambda}=\frac13 (\sqrt{2}-1)\approx0.1381$. The minimal measurement dependence needed to explain the maximal quantum violation of $I_3\leq 0$ for the uniformly distributed instrumental variables is $\mathcal{M}_{X:\Lambda}=\left(\frac{1}{\sqrt{2}}-\frac12 \right)\approx0.2071$.

Finally, we consider the quantum violation of the causal bounds for the instrument that takes three values. The inequality $\mathrm{ACE}_{A\rightarrow B}\geq C_2$ can be violated by the maximally entangled state  with the amount of $\left(\frac{1}{\sqrt{2}}-\frac12 \right)\approx0.2071$ 
and the inequality $\mathrm{ACE}_{A\rightarrow B}\geq C_3$ with the amount of $ \sqrt2-1\approx 0.4142.$  In order to explain these violations, the amount of minimum measurement dependency in the classical instrumental causal structure depends on a probability distribution of the instrumental random variable. See Fig.~\ref{fig: Beyond_binary} ~(center) and (right) for the particular examples of such distributions. We highlight that even though the violations of causal bounds match with the violations of instrumental inequalities, these quantities are of a very different nature. In particular, the violation of causal bounds required both interventional and observational probability distributions while the violation of instrumental inequalities rely solely on observational data.

\section{Discussion}
\label{sec5}
Instrumental variables offer ways to estimate causal influence even under confounding effects and without the need for interventions. Strikingly, as discovered in Ref.~\cite{Balke1997}, one can infer the effect of interventions, without resorting to any structural equations, simply from observational data obtained with the help of an instrument. As already recognized long ago \cite{johnston1963econometric}, however, ``the real difficulty in practice of course is actually finding variables to play the role of instruments''. Since the potential correlation of the instrument with any latent variables is in principle unobservable, it might seem that the exogeneity of a given instrument is a matter of trust and intuition rather than a fact supported by the data.

Motivated by this fundamental problem, the data from an instrumental test  \cite{pearl1995testability,bonet2013instrumentality,Kedagni2020,poderini2020exclusivity} can be employed to benchmark the amount of dependence the instrument can have with a confounding variable. More precisely, we quantify such correlations via a $l1$-norm, measuring by how much the instrumental variable fails to be exogenous. The violation of an instrumental inequality allows us then to put lower bounds on this dependence.
In turn, we derive bounds for the average causal effect \cite{pearl2009causality} taking into account that some level of dependence, lower bounded by the violation of instrumental inequality, is present. That is, we turn the causal bounds in a reliable tool even if the instrument is not really exogenous.

Relying on a linear program description, we obtain fully analytical results for the simplest instrumental scenario where all variables are binary. We study a more general case of trinary instrumental variable numerically using our linear programming technique. In parallel, we also derived new bounds for the average causal effect (Eq. \eqref{eq:ace34}), that to the best of our knowledge, are new to the literature. We also consider applications of our generalized instrumental inequalities and causal bounds to consider the problem of measurement independence (also known as ``free-will'') in the foundations of quantum physics.

It is worth noting that the effect of imperfect instruments has previously been considered~\cite{bartels1991instrumental}. There, however, the study was limited to regression bivariate models, while here our results are free of any structural equations and valid for any causal mechanisms between the variables. Even though, we have focused on the case where treatment and effect variables are binary, the linear program framework we propose can also be extended to variables assuming any discrete number of values (limited, of course,  by the computational complexity of the problem). Another interesting question for future research is to understand whether similar results may hold for the case of continuous variables, a direction that we hope might be triggered by our results.

\begin{acknowledgments}
 N.M.~acknowledges the support by the Foundation for Polish Science (IRAP project,
ICTQT, contract no. 2018/MAB/5, co-financed by EU
within Smart Growth Operational Programme) and
the Deutsche Forschungsgemeinschaft (DFG, German
Research Foundation) via the Emmy Noether grant
441423094.
This work was supported by the John Templeton Foundation via the grant Q-CAUSAL No 61084 (the opinions expressed in this publication are those of the author(s) and do not necessarily reflect the views of the John Templeton Foundation) Grant Agreement No. 61466, by the Serrapilheira Institute (grant number Serra – 1708-15763), by the Simons Foundation (Grant Number 884966, AF), the Brazilian National Council for Scientific and Technological Development (CNPq) via the National Institute for Science and Technology on Quantum Information (INCT-IQ) and Grants No. 406574/2018-9 and 307295/2020-6, the Brazilian agencies MCTIC and MEC. M.G.~is funded by the Deutsche Forschungsgemeinschaft (DFG, German Research Foundation) under Germany’s Excellence Strategy – Cluster of Excellence Matter and Light for Quantum Computing
(ML4Q) EXC 2004/1 – 390534769.
\end{acknowledgments}

\bibliography{ref}
\end{document}